\documentclass[journal,twoside,web]{IEEEtran}
\usepackage{cite}
\usepackage{amsthm}
\usepackage{amsmath,amssymb,amsfonts}
\usepackage{algorithmic}
\usepackage{graphicx}
\usepackage{textcomp}

\usepackage{graphics} 
\usepackage{epsfig} 
\usepackage{mathptmx} 
\usepackage{times} 
\usepackage{amsmath} 
\usepackage{amssymb}  
\newcommand{\N}{\mathcal{N}}
\usepackage{algorithm}
\usepackage{algorithmic}
\usepackage{graphicx}
\usepackage{wrapfig}
\usepackage{cancel}
\usepackage{subfigure}
\usepackage{url}
\usepackage{hyperref}
\usepackage{color}

\newtheorem{theorem}{Theorem}

\newtheorem{remark}{Remark}
\newtheorem{corollary}{Corollary}
\newtheorem{definition}{Definition}
\newtheorem{conjecture}{Conjecture}

\makeatletter
\newcommand\fs@betterruled{%
  \def\@fs@cfont{\bfseries}\let\@fs@capt\floatc@ruled
  \def\@fs@pre{\vspace*{5pt}\hrule height.8pt depth0pt \kern2pt}%
  \def\@fs@post{\kern2pt\hrule\relax}%
  \def\@fs@mid{\kern2pt\hrule\kern2pt}%
  \let\@fs@iftopcapt\iftrue}
\floatstyle{betterruled}
\restylefloat{algorithm}
\makeatother

\usepackage{tikz}

\newcommand\copyrighttext{%
  \footnotesize \textcopyright 2021 IEEE.  Personal use of this material is permitted.  Permission from IEEE must be obtained for all other uses, in any current or future media, including reprinting/republishing this material for advertising or promotional purposes, creating new collective works, for resale or redistribution to servers or lists, or reuse of any copyrighted component of this work in other works.
  
  Citation information: DOI 10.1109/LCSYS.2021.3082875, IEEE Control Systems Letters}
\newcommand\copyrightnotice{%
\begin{tikzpicture}[remember picture,overlay]
\node[anchor=south,yshift=-50pt] at (current page.north) {\fbox{\parbox{\dimexpr\textwidth-\fboxsep-\fboxrule\relax}{\copyrighttext}}};
\end{tikzpicture}%
}

\def\BibTeX{{\rm B\kern-.05em{\sc i\kern-.025em b}\kern-.08em
    T\kern-.1667em\lower.7ex\hbox{E}\kern-.125emX}}
\begin{document}

\title{A Stable High-order Tuner for General Convex Functions}
\author{Jos{\'e} M. Moreu and Anuradha M. Annaswamy
\thanks{This work was supported in part by the US-Spain Fulbright Commission $\&$ Cosentino S.A. and by the Boeing Strategic University Initiative.}
\thanks{J.M. Moreu is with the Department of Mechanical Engineering, Massachusetts Institute of Technology,
        Cambridge, MA 02139, USA
        {\tt\small jmmoreu@mit.edu}}
\thanks{A.M. Annaswamy is Senior Research Scientist in the Department of Mechanical Engineering, Massachusetts Institute of Technology, Cambridge, MA 02139, USA
        {\tt\small aanna@mit.edu}}
}

\IEEEoverridecommandlockouts
\IEEEpubid{\makebox[\columnwidth]{978-1-5386-5541-2/18/\$31.00~\copyright2021 IEEE \hfill} \hspace{\columnsep}\makebox[\columnwidth]{ }}
\maketitle
\thispagestyle{empty} 
\IEEEpubidadjcol

\begin{abstract}
Iterative gradient-based algorithms have been increasingly applied for the training of a broad variety of machine learning models including large neural-nets. In particular, momentum-based methods, with accelerated learning guarantees, have received a lot of attention due to their provable guarantees of fast learning in certain classes of problems and multiple algorithms have been derived. However, properties for these methods hold only for constant regressors. When time-varying regressors occur, which is commonplace in dynamic systems, many of these momentum-based methods cannot guarantee stability. Recently, a new High-order Tuner (HT) was developed for linear regression problems and shown to have 1) stability and asymptotic convergence for time-varying regressors and 2) non-asymptotic accelerated learning guarantees for constant regressors.  In this paper, we extend and discuss the results of this same HT for general convex loss functions. Through the exploitation of convexity and smoothness definitions, we establish similar stability and asymptotic convergence guarantees. Finally, we provide numerical simulations supporting the satisfactory behavior of the HT algorithm as well as an accelerated learning property.
\end{abstract}

\copyrightnotice
\section{Introduction}

\IEEEPARstart{G}{radient-descent} constitutes the nerve center of solutions to several problems in a wide range of fields such as adaptive control, machine learning, and optimization  \cite{Narendra_1989, boyd2004convex, hazan2019introduction, Nesterov2018LecturesOptimization}. In adaptive control, reducing the control tracking error of an uncertain dynamical system as well as learning the  unknown parameters of the system are the underlying goals. A gradient-descent approach is often employed to realize both goals, first to obtain a fast convergence of the performance error and then, to reduce the learning error. In machine learning, fast and correct training of models such as neural networks is sought after, which necessitates the reduction of an underlying loss function using a gradient-based approach.  Optimization approaches require the solution of an augmented Lagrangian in an expedient manner, through a gradient-descent method. Given the importance of the fast convergence in all these problems, there is a need for algorithms that can lead to an order of magnitude improvement in the speed of convergence, both performance and learning errors, while retaining stability. This paper proposes such an algorithm.

Recently, a class of High-order Tuners (HT) was proposed in continuous-time \cite{Gaudio2020ASystems} for a large class of dynamic systems for the purpose of estimation of unknown parameters. The estimation problem for this class of systems can be reformulated as a linear regression problem, where the underlying regressors correspond to various system variables that can be measured, including inputs, outputs, and states. Each of these high-order tuners was shown to result in a stable performance error when the regressors were time-varying. A variational perspective was proposed in \cite{Gaudio2020ASystems} with a Hamiltonian interpretation as the unifying framework for this class of high-order tuners. One of these tuners was extended in \cite{boffi2020implicit} for a class of nonlinear problems where the underlying error model is still based on linear regression. These high-order tuners are inspired by earlier work in \cite{morse1992high} and \cite{evesque2003adaptive}.

The main motivation of these high-order tuners was to speed up the performance that could be obtained from gradient methods. In \cite{Gaudio2020AcceleratedRegressors}, a discrete-time HT was proposed and was shown to have two important properties. First, the HT was shown to have  accelerated convergence of performance error when the regressors are constant, with the rate of convergence a log factor away from the well known Nesterov’s algorithm \cite{Nesterov2018LecturesOptimization}.  Second, it was shown to be stable, whether or not the underlying regressors are constant or time-varying.  In contrast, Nesterov’s algorithm becomes unstable for the time-varying regressor case. All of these discussions were limited to linear error models which in turn have a quadratic and homogeneous loss function. 

In this paper, we extend the results of our HT in \cite{Gaudio2020AcceleratedRegressors} for discrete-time systems with convex loss functions and therefore applicable to a large class of error models. We show both for the case when the loss function is smooth and convex, and for the case when it is smooth and strongly convex that the HT can be guaranteed to be stable. In the first case, we conclude boundedness of the parameter estimate and that the loss function reaches its minimum (Theorem \ref{theorem:HOT_smooth}), while in the second case, we establish exponential convergence of the parameter estimate to its true value and exponential convergence of the loss function towards its minimum (Theorem \ref{theorem:HOT_strongly}). As a precursor to both these cases, we consider a HT for a continuous-time systems with convex functions, and establish a similar stability result to Theorem \ref{theorem:HOT_smooth} (Theorem \ref{theorem:HT_continuous}).  For ease of exposition, we repeat the results of \cite{Gaudio2020AcceleratedRegressors}. We show through simulations that our HT leads to an accelerated convergence of the performance error for a general convex function.

The organization of the paper is as follows. In Section \ref{section:preliminaries}, we present a few preliminaries related to convex functions as well as the problem statement. In Section \ref{section:linear}, we focus on the minimization of smooth convex functions. We first present existing results related to HT for linear error models presented in \cite{Gaudio2020AcceleratedRegressors}. We then present our first main contribution of the paper, where we extend the stability properties of these HT for general convex loss functions, first in continuous time, and then in discrete-time. In Section \ref{section:SmoothCase} we present the stability properties of HT for smooth and strongly convex functions. In Section \ref{section:sims}, several numerical examples are discussed which illustrate the convergence properties of the proposed high-order tuner.

\section{Preliminaries}\label{section:preliminaries}

\subsection{Definitions}
The following definitions and properties will be used throughout this paper, modified from \cite{Nesterov2018LecturesOptimization, Bubeck2015ConvexComplexity}.
\begin{definition}\label{def:convexity} 
A continuously differentiable function $f$ is convex if
\begin{equation}\label{eq:convexity}
    f(y)\geq f(x)+\nabla f(x)^T(y-x), \quad \forall x,y\in\mathbb{R}^N.
\end{equation}
\end{definition}
\begin{definition}\label{def:strgly-convexity} 
A continuously differentiable function $f$ is $\mu$-strongly convex if there exists a $\mu>0$ such that
\begin{equation}\label{eq:strgly-convexity}
    f(y)\geq f(x)+\nabla f(x)^T(y-x)+\frac{\mu}{2}\lVert y-x\rVert^2, \quad \forall x,y\in\mathbb{R}^N.
\end{equation}
\end{definition}
\begin{definition}\label{def:smoothness} 
A continuously differentiable function $f$ is $\bar{L}$-smooth if there exists a $\bar{L}>0$ such that
\begin{equation}\label{eq:smoothness}
    f(y)\leq f(x)+\nabla f(x)^T(y-x)+\frac{\bar{L}}{2}\lVert y-x\rVert^2, \quad \forall x,y\in\mathbb{R}^N.
\end{equation}
\end{definition}
\begin{corollary}\label{def:a1def} A continuously differentiable function $f$ is convex and $\bar{L}$-smooth if there exists a $\bar{L}>0$ such that $\forall x,y\in\mathbb{R}^N$,
\begin{equation*}\label{eq:smooth_convex}
    \nabla f(x)^T(y-x)\leq  f(y)-f(x)\leq \nabla f(x)^T(y-x)+\frac{\bar{L}}{2}\lVert y-x\rVert^2.
\end{equation*}
\end{corollary}
\begin{corollary}\label{def:a2def} A continuously differentiable function $f$ is $\mu$-strongly convex and $\bar{L}$-smooth if there exists two scalars $\bar{L}\geq\mu>0$ such that $\forall x,y\in\mathbb{R}^N$,
\begin{equation*}\label{eq:smooth_strongly_convex}
        \frac{\mu}{2}\lVert y-x\rVert^2 \leq  f(y)-f(x)-\nabla f(x)^T(y-x)\leq \frac{\bar{L}}{2}\lVert y-x\rVert^2.
\end{equation*}
\end{corollary}
For ease of exposition, all convex functions considered in this paper satisfy Corollary \ref{def:a1def}, and all strongly convex functions satisfy Corollary \ref{def:a2def}.

\subsection{Problem Statement}
The focus of this paper is on the minimization problem defined as
\begin{equation} \label{eq:smooth_loss}
    \min_{\theta\in\mathbb{R}^N}L_k(\theta),
\end{equation}
where $L(\theta)$ is a convex loss function that depends on time-varying regressors $\phi_k$, and denotes the value obtained with $\theta$ at iteration $k$. The goal is to carry out a recursive estimation of $\theta^*$, the solution of \eqref{eq:smooth_loss}, so that the estimate $\theta_k$ quickly converges to $\theta^*$.  We provide a solution to this minimization problem for smooth loss functions in Section \ref{section:linear} and for smooth and strongly convex loss functions in Section \ref{section:SmoothCase}.

\section{Smooth Loss Functions}\label{section:linear}
\subsection{Quadratic Minimization}
We first consider the case when $L(\theta)$ is a smooth and quadratic loss function of $\theta$. Such problems are ubiquitous in several applications of optimization, estimation, and control and arise in linear regression problems. The underlying problem is of the form
\begin{equation} \label{eq:linear_reg}
y_k=\theta^{*T}\phi_k \end{equation}
where $\phi_k \in \mathbb{R}^N$ is a regressor that varies with time,  $y_k \in \mathbb{R}$ is a measurable output, and $\theta^*\in\mathbb{R}^N$ denotes the unknown parameter and needs to be estimated. Using the structure of the linear model in \eqref{eq:linear_reg}, an estimator is formulated as $\hat{y}_k=\theta_k^T\phi_k$, where $\hat{y}_k$ is the output estimate and $\theta_k\in\mathbb{R}^N$ is the parameter estimate. This leads to a performance error
\begin{equation}\label{eq:linear_reg_error}
    e_{y,k}=\hat{y}_k-y_k=\tilde{\theta}_k^T\phi_k,
\end{equation}
where $\tilde{\theta}_k=\theta_k-\theta^*$. It can be seen that this in turn leads to a loss function $L_k(\theta_k)$,
\begin{equation}\label{eq:linear_reg_loss}
    L_k(\theta_k)=\frac{1}{2}e_{y,k}^2=\frac{1}{2}\tilde{\theta}_k^T\phi_k\phi_k^T\tilde{\theta}_k,
\end{equation} that is quadratic in $\theta_k$ (\cite{Gaudio2020AcceleratedRegressors}).
The gradient of this loss function is implementable as $\nabla L_k(\theta_k)=\phi_k e_{y,k}=\phi_k \phi_k^T\tilde{\theta}_k$.

We note that the Hessian of \eqref{eq:linear_reg_loss}, $\nabla^2 L_k(\theta_k)=\phi_k\phi_k^T$, can be bounded as $0\leq\nabla^2 L_k(\theta_k)\leq\lVert\phi_k\rVert^2I$. Also, we note that $\nabla L_k(\theta^*)=0$. Therefore, the quadratic loss in \eqref{eq:linear_reg_loss} is a convex function, that need not be strongly convex, with a time-varying and regressor-dependent $\bar{L}$-smoothness parameter.

It is well known that stable parameter estimation and control can be enabled by utilizing a normalized gradient descent method given by \cite[Chapter~3]{Goodwin1984}:
\begin{equation}\label{eq:Norm_GD}
    \theta_{k+1}=\theta_k-\bar{\alpha}\nabla \bar{f}_k(\theta_k),\quad 0<\bar{\alpha}<2,
\end{equation}
where $\bar{f}_k(\cdot)$ is the normalized loss function defined as
\begin{equation}\label{eq:Normalized_loss}
    \bar{f}_k(\theta_k)=\frac{L_k(\theta_k)}{\N_k},
\end{equation}
and $\N_k=1+\lVert\phi_k\rVert^2$ is a normalizing signal. Note that $\nabla^2 \bar{f}_k(\theta_k)=\nabla^2 L_k(\theta_k)/\N_k\leq I$, and therefore \eqref{eq:Normalized_loss} is a $1$-smooth convex function. We refer the reader to \cite[Chapter~3]{Goodwin1984} for details of stability and convergence.

Rather than such a first-order tuner in \eqref{eq:Norm_GD}, second-order tuners were proposed in \cite{Gaudio2020AcceleratedRegressors} for the same discrete-time regression problem in \eqref{eq:linear_reg} (and in \cite{Gaudio2020ASystems} for continuous-time problems with linear parameterization), and shown to lead to stability and accelerated convergence of the performance error $e_{y,k}$. This high-order tuner is summarized below as Algorithm \ref{alg:HOT_1}.

\vspace{-0.5cm}
\begin{algorithm}[H]
\caption{\strut HT Optimizer for Linear Regression}
\label{alg:HOT_1}
\begin{algorithmic}[1]
\STATE {\bfseries Input:} initial conditions $\theta_0$, $\vartheta_0$, gains $\gamma$, $\beta$
\FOR{$k=0,1,2,\ldots$}
\STATE \textbf{Receive} regressor $\phi_k$, output $y_k$
\STATE Let $\N_k=1+\lVert\phi_k\rVert^2$, $\nabla L_k(\theta_k)=\phi_k(\theta_k^T\phi_k-y_k)$,\\
$\nabla \bar{f}_k(\theta_k)=\frac{\nabla L_k(\theta_k)}{\N_k}$,\\ $\bar{\theta}_k=\theta_k-\gamma\beta\nabla \bar{f}_k(\theta_k)$
\STATE $\theta_{k+1}\leftarrow\bar{\theta}_k-\beta(\bar{\theta}_k-\vartheta_k)$
\STATE Let $\nabla L_k(\theta_{k+1})=\phi_k(\theta_{k+1}^T\phi_k-y_k)$,\\
$\nabla \bar{f}_k(\theta_{k+1})=\frac{\nabla L_k(\theta_{k+1})}{\N_k}$
\STATE $\vartheta_{k+1}\leftarrow\vartheta_k-\gamma\nabla \bar{f}_k(\theta_{k+1})$
\ENDFOR
\end{algorithmic}
\end{algorithm}
\vspace{-0.7cm}

\subsection{Extension to General Smooth Convex Minimization }
We  present the first set of main results of this paper in this section, which addresses \eqref{eq:smooth_loss} for the case when $L(\theta)$ is a smooth convex function of $\theta$. We propose a high-order tuner (HT) for the adjustment of $\theta_k$. This HT is similar to Algorithm \ref{alg:HOT_1}, but with the distinction that the normalizing signal $\N_k$ is chosen as
\begin{equation}
    \N_k = 1+H_k,
\end{equation}
where 
\begin{equation} \label{eq:H_k_definition}
    H_k = \max\left\{ \lambda:\lambda\in \sigma\left(\nabla^2 L_k(\theta)\right)\right\},
\end{equation}
and $\sigma\left(\nabla^2 L_k(\theta)\right)$ denotes the spectrum of the Hessian matrix of the loss function $L_k$. 

\vspace{-0.5cm}
\begin{algorithm}[H]
\caption{HT Optimizer for Convex Loss Functions}
\label{alg:HOT_2}
\begin{algorithmic}[1]
\STATE {\bfseries Input:} initial conditions $\theta_0$, $\vartheta_0$, gains $\gamma$, $\beta$
\FOR{$k=0,1,2,\ldots$}
\STATE \textbf{Receive} regressor $\phi_k$,
\STATE Compute $\nabla L_k(\theta_k)$ and let $\N_k=1+H_k$,\\
$\nabla \Bar{f}_k(\theta_k)=\frac{\nabla L_k(\theta_k)}{\N_k}$,\\ 
$\bar{\theta}_k=\theta_k-\gamma\beta\nabla \Bar{f}_k(\theta_k)$
\STATE $\theta_{k+1}\leftarrow\bar{\theta}_k-\beta(\bar{\theta}_k-\vartheta_k)$
\STATE Compute $\nabla L_k(\theta_{k+1})$ and let\\
$\nabla \Bar{f}_k(\theta_{k+1})=\frac{\nabla L_k(\theta_{k+1})}{\N_k}$,\\ $\vartheta_{k+1}\leftarrow\vartheta_k-\gamma\nabla \Bar{f}_k(\theta_{k+1})$
\ENDFOR
\end{algorithmic}
\end{algorithm}
\vspace{-0.5cm}

Algorithm \ref{alg:HOT_2} implies that at each $k$, in addition to the first moment, $\nabla L_k$ which is evaluated both at $\theta_k$ and $\theta_{k+1}$, we also have access to the second moment, $H_k$, as well, in the form of \eqref{eq:H_k_definition}. In many engineering problems, the underlying model includes prior information regarding the causality of the loss function \cite{Annaswamy1998ApplicationsParametrization}. Both the first and second moment may be functions of the regressor, with $\nabla L_k=g(\phi_k)$, and  $H_k=h(\phi_k)$, where $\phi_k$ is the value of the regressor during iteration $k$. The prior information may then imply that $g(\cdot)$ and $h(\cdot)$ are known functions, allowing Algorithm \ref{alg:HOT_2} to be implemented at each $k$. If these functions are poorly known, then conservative choices have to be made in the implementations of $\nabla L_k$ and $H_k$. For example, the second moment can be chosen as $H_k=\bar{L}$ where $\bar{L}$ is the smoothness parameter of $L_k$.

%
As many of the tools and methods adopted for proving stability have their origins in continuous time, we first address the counterpart of Algorithm \ref{alg:HOT_2} in continuous time and its stability property. 

\subsubsection{Continuous time stability}
The problem under consideration is the 
determination of $\theta^*$, which is the solution of 
\begin{equation} \label{eq:smooth_loss_continuous}
    \min_{\theta(t)\in\mathbb{R}^N}L_t(\theta(t)),
\end{equation}
where $L_t(\theta(t))$, is the loss function obtained with $\theta(t)$ at iteration $t$, which varies continuously. This problem has been studied in detail in \cite{Gaudio2020ASystems} for linear models that are static and dynamic, with a unifying variational perspective. A class of HT tuners was proposed in continuous-time, all of which were shown to be stable. In what follows, we extend the stability results in \cite{Gaudio2020ASystems} for a general model that leads to a convex loss function.

One of the high-order tuners among the class of continuous-time algorithms proposed in  \cite{Gaudio2020ASystems} is given by
\begin{subequations}\label{eq:Continuous_HOT}
    \begin{align}
        \dot{\vartheta}(t)&=-\frac{\gamma}{\N_t}\nabla L_t(\theta(t)),\label{eq:Continuous_HOT_1}\\
        \dot{\theta}(t)&=-\beta(\theta(t)-\vartheta(t))\label{eq:Continuous_HOT_2}.
    \end{align}
\end{subequations}
where $\N_t=1+ H_t$, where $H_t$ is the continuous-time equivalent of \eqref{eq:H_k_definition} and was shown to be stable for linear regression models. Yet another tuner was proposed in \cite{Gaudio2020ASystems}  with the  signal $\N_t$ appearing in the numerator rather than the denominator, subsequently expanded in \cite{boffi2020implicit} for a class of nonlinearly parameterized systems that has an equivalent model that can be linearly parameterized. The stability property of the HT in \eqref{eq:Continuous_HOT_1}-\eqref{eq:Continuous_HOT_2} is summarized in Theorem \ref{theorem:HT_continuous}.

\begin{theorem}\label{theorem:HT_continuous}
  For a continuously differentiable convex loss function $L$, the continuous time HT in \eqref{eq:Continuous_HOT_1}-\eqref{eq:Continuous_HOT_2}, with $\beta> 2\gamma > 0$, ensures that $V = \frac{1}{\gamma}\lVert\vartheta-\theta^*\rVert^2+\frac{1}{\gamma}\lVert\theta-\vartheta\rVert^2$ is a Lyapunov function.
\end{theorem}
\begin{proof} 
    First, consider the following expression for smooth and convex functions \cite[Lemma~3.5]{Bubeck2015ConvexComplexity}:
    \begin{equation}\label{eq:smooth_lower_bound}
        L_t(\theta)-L_t(\theta^*)+\frac{1}{2\bar{L}_t}\lVert\nabla L_t(\theta)\rVert^2\leq\nabla L_t(\theta)^T\tilde{\theta}.
    \end{equation}
    Using \eqref{eq:Continuous_HOT}, \eqref{eq:smooth_lower_bound} and $\gamma\leq\beta/2$, the time derivative of $V$ may be bounded as $\dot{V}\leq\frac{1}{\N_t}\{-2 \left(L_t(\theta)-L_t(\theta^*)\right)- \frac{2}{\gamma}\beta\lVert\theta-\vartheta\rVert^2-\left[\frac{1}{\sqrt{\bar{L}_t}}\lVert\nabla L_t(\theta)\rVert-2\sqrt{\bar{L}_t}\lVert\theta-\vartheta\rVert \right]^2 \}\leq 0 $.
    Thus it can be concluded that $V$ is a Lyapunov function with $(\vartheta-\theta^*)\in\mathcal{L}_{\infty}$ and $(\theta-\vartheta)\in\mathcal{L}_{\infty}$. Integrating $\dot{V}$ from $t_0$ to $\infty$: $\int^\infty_{t_0} (L(\theta)-L(\theta^*))/\N_tdt\leq-\int^\infty_{t_0}\dot{V}dt=V(t_0)-V(\infty)<\infty$, thus $L(\theta)-L(\theta^*)\in\mathcal{L}_1\cap\mathcal{L}_\infty$.
\end{proof}

\subsubsection{Discrete time stability}
We now proceed with the stability property of the HT Algorithm \ref{alg:HOT_2}. 
\begin{theorem}\label{theorem:HOT_smooth}
  For a continuously differentiable $\bar{L}_k$-smooth convex loss function $L_k(\cdot)$, Algorithm \ref{alg:HOT_2}, with $0<\beta<1$ and $0<\gamma\leq\frac{\beta(2-\beta)}{8+\beta}$, ensures that $V = \frac{1}{\gamma}\lVert\vartheta-\theta^*\rVert^2+\frac{1}{\gamma}\lVert\theta-\vartheta\rVert^2$ is a Lyapunov function. It can also be shown that $\lim_{k\rightarrow\infty}L_k(\theta_{k+1})-L_k(\theta^*)=0$.
\end{theorem}
\begin{proof}
    First, through convexity and smoothness definitions, as well as the structure of the HT, one can obtain the following upper bound:
    \begin{equation*}
        L_k(\vartheta_k)-L_k(\bar{\theta}_k)=L_k(\vartheta_k)-L_k(\theta_{k+1})+L_k(\theta_{k+1})-L_k(\bar{\theta}_k)
    \end{equation*}
    \vspace{0cm}
    \begin{equation}\label{eq:two}
        \begin{split}
            &\overset{\eqref{eq:convexity}, \eqref{eq:smoothness} }{\leq} \nabla L_k(\theta_{k+1})^T(\vartheta_{k}-\theta_{k+1})+\frac{\bar{L}_k}{2}\lVert\vartheta_{k}-\theta_{k+1}\rVert^2\\
            &\quad +\nabla L_k(\theta_{k+1})^T(\theta_{k+1}-\bar{\theta}_k)
        \end{split}
    \end{equation}
        \begin{equation}\label{eq:three}
        \begin{split}
            &\overset{\text{Alg. } \ref{alg:HOT_2}}{\leq} \nabla L_k(\theta_{k+1})^T(\vartheta_{k}-\bar{\theta}_k)+\frac{\bar{L}_k}{2}\lVert\vartheta_{k}-(1-\beta)\bar{\theta}_k-\beta\vartheta_k)\rVert^2
        \end{split}
    \end{equation}
    \begin{equation} \label{eq:UB_result1}
        \begin{split}
            &L_k(\vartheta_k)-L_k(\bar{\theta}_k)\\
            &\leq -\nabla L_k(\theta_{k+1})^T(\bar{\theta}_k-\vartheta_{k})+\frac{\bar{L}_k}{2}(1-\beta)^2\lVert\bar{\theta}_k-\vartheta_k\rVert^2.
        \end{split}
    \end{equation}
    Similarly, we obtain:
    \begin{equation} \label{eq:UB_result2}
        \begin{split}
            &L_k(\bar{\theta}_k)-L_k(\vartheta_k)\\
            &\quad\leq \nabla L_k(\theta_{k})^T(\bar{\theta}_k-\vartheta_{k})+\frac{\bar{L}_k\gamma^2\beta^2}{2\N_k^2}\lVert\nabla L_k(\theta_k)\rVert^2.
        \end{split}
    \end{equation}
    Flipping the signs in \eqref{eq:UB_result1} and using \eqref{eq:UB_result2} we obtain:
    \begin{equation}\label{eq:mainresult_1}
        \begin{split}
            &\nabla L_k(\theta_{k+1})^T(\bar{\theta}_k-\vartheta_{k})\\
            &-\frac{\bar{L}_k}{2}(1-\beta)^2\lVert\bar{\theta}_k-\vartheta_k\rVert^2\\
            &-\frac{\bar{L}_k\gamma^2\beta^2}{2\N_k^2}\lVert\nabla L_k(\theta_k)\rVert^2\leq \nabla L_k(\theta_{k})^T(\bar{\theta}_k-\vartheta_{k})
        \end{split}
    \end{equation}
    
    Using Algorithm \ref{alg:HOT_2}, \eqref{eq:smooth_lower_bound}, \eqref{eq:mainresult_1},  setting $\gamma\leq\frac{\beta(2-\beta)}{8+\beta}$, and defining $\Delta V_k:=V_{k+1}-V_k$, it can be shown that    \begin{equation}\label{eq:key_step}
           \Delta V_k\leq\frac{1}{\N_k}\{-2(L_k(\theta_{k+1})-L_k(\theta^*))\} \leq 0.
    \end{equation}
Collecting $\Delta V_k$ terms from $t_0$ to $T$, and letting $T\rightarrow \infty$, it can be seen that  $L_k(\theta_{k+1})-L_k(\theta^*)
\in\ell_1\cap\ell_\infty$ and therefore $\lim_{k\rightarrow\infty}L_k(\theta_{k+1})
-L_k(\theta^*)=0.$
\end{proof}

\begin{remark}
The main idea behind the use of the Hessian in $\N_k$  stems from the fundamental property of convex functions (see \eqref{eq:smooth_lower_bound}) for continuous-time and \eqref{eq:UB_result1} and \eqref{eq:UB_result2} for discrete-time cases) that are most relevant to establish the underlying Lyapunov functions (Theorem \ref{theorem:HT_continuous} for continuous-time and Theorem \ref{theorem:HOT_smooth} for discrete-time). That is, convexity of a function allows an important inequality that connects its first-order moment (involving gradients) with its second-order moment (involving Hessians). This inequality in turn necessitates normalization that involves Hessian in the loss function.
\end{remark}

\begin{remark}
It should be noted that the proof of Theorem \ref{theorem:HOT_smooth} is highly nontrivial. In particular, the derivation of \eqref{eq:key_step} was enabled through a careful deployment of properties of convex functions. In particular, the inequality \eqref{eq:UB_result1} was arrived at using the smoothness property and the convexity property of convex functions in \eqref{eq:two}, and the specific structure of the high-order tuner in \eqref{eq:three}. All these three components were equally central in deriving \eqref{eq:UB_result1} and \eqref{eq:UB_result2} and therefore the final step in \eqref{eq:key_step}. It should be noted that unlike this convex case considered in this paper, the result in \cite{Gaudio2020AcceleratedRegressors} relied on the fact that the gradient is \underline{linear in} $||{\theta}_k||$. Such a property does not hold in the current context. It is only through the use of properties of general convex functions were we able to establish \eqref{eq:key_step}.
\end{remark}

\section{Smooth and Strongly-Convex Loss Functions}\label{section:SmoothCase}
In the previous section we addressed the case when the loss function $L(\theta)$ was smooth and convex in $\theta$ and derived stability properties of HT in both continuous and discrete-time in Theorems \ref{theorem:HT_continuous} and \ref{theorem:HOT_smooth}, respectively. In this section, we consider loss functions $L(\theta)$ that are smooth and strongly convex, restrict our attention to discrete-time problems, and propose the same HT as in Algorithm 2. In addition to deriving stability properties, we also show that Algorithm 2 leads to an accelerated convergence of $L(\theta_k)$ to zero with constant regressors, i.e. $\phi_k\equiv \phi$. 
\vspace{-0.3cm}
\subsection{Quadratic Minimization}
As in Section \ref{section:linear}, we first consider the simple case when $L(\theta)$ is quadratic in $\theta$, which is given by \eqref{eq:linear_reg_loss}. As $L(\cdot)$ is only convex and not strongly-convex, a regularizing term is added to the normalized loss function \eqref{eq:Normalized_loss} as in
\begin{equation} \label{eq:augmented_function}
    f_k(\theta_k)=\frac{L_k(\theta_k)}{\N_k}+\frac{\mu}{2}\lVert\theta_k-\theta_0\rVert^2,
\end{equation}
where $\mu>0$ is a regularization constant and $\theta_0$ is the initial condition of the estimate \cite{Gaudio2020AcceleratedRegressors}. It can be seen that $f$ is a $(1+\mu)$-smooth and $\mu$-strongly convex function since $\mu I\leq \nabla^2 f_k(\theta_k)\leq (1+\mu)I$. The optimal solution of $f_k(\cdot)$ is defined as $\theta_\epsilon^*$, i.e.: $\nabla f_k(\theta_\epsilon^*)=0$. 

\begin{remark}
Algorithm \ref{alg:HOT_1} can be applied in this setting by simply replacing $\nabla \bar{f}_k(\theta_k)$ with the gradient of \eqref{eq:augmented_function}, i.e.: $\nabla f_k(\theta_k)=\frac{\nabla L_k(\theta_k)}{\N_k}+\mu(\theta_k-\theta_0)$. See \cite{Gaudio2020AcceleratedRegressors} for further details.
\end{remark}

\subsection{Minimization of Smooth and Strongly Convex Functions}
We now address the problem when $L$ is a smooth and strongly convex function. We first present the stability property of the HT and then its accelerated convergence.
\subsubsection{Discrete time stability}
\begin{theorem}\label{theorem:HOT_strongly}
  For a continuously differentiable $\bar{L}_k$-smooth and $\mu$-strongly convex loss function $L_k(\cdot)$, Algorithm \eqref{alg:HOT_2}, with $0<\beta<1$ and $0<\gamma\leq\frac{\beta(2-\beta)}{16+\beta+\mu}$ ensures that $V = \frac{1}{\gamma}\lVert\vartheta-\theta^*\rVert^2+\frac{1}{\gamma}\lVert\theta-\vartheta\rVert^2$ is a Lyapunov function. It can also be shown that $\lim_{k\rightarrow\infty}(L_k(\theta_{k+1})-L_k(\theta^*))=0$. Furthermore, for constant regressors, with $\N_k=\N$, $V_k\leq\exp{(-\mu Ck)}V_0$, where $C=\frac{\gamma\beta}{4\N}$.
\end{theorem}
\begin{proof}
    As $L(\theta)$ is strongly convex in $\theta$, we utilize inequality \eqref{eq:strgly-convexity} rather than \eqref{eq:convexity} which leads us to the following:
    \begin{equation}\label{eq:mainresult_2}
        \begin{split}
            &\nabla L_k(\theta_{k+1})^T(\bar{\theta}_k-\vartheta_{k})-\frac{\bar{L}_k}{2}(1-\beta)^2\lVert\bar{\theta}_k-\vartheta_k\rVert^2 \\
            &-\frac{\bar{L}_k\gamma^2\beta^2}{2\N_k^2}\lVert\nabla L_k(\theta_k)\rVert^2+\frac{\mu}{2}\beta^2\lVert\bar{\theta}_k-\vartheta_k\rVert^2\\
            &+\frac{\mu}{2}\lVert \theta_k-\vartheta_k \rVert^2\leq \nabla L_k(\theta_{k})^T(\bar{\theta}_k-\vartheta_{k})
        \end{split}
    \end{equation}
    Using Algorithm \ref{alg:HOT_2}, \eqref{eq:smooth_lower_bound}, \eqref{eq:mainresult_2}, setting $\gamma\leq\frac{\beta(2-\beta)}{16+\beta+\mu}$ and defining $\Delta V_k:=V_{k+1}-V_k$ it can be shown that 
    \begin{equation}\label{eq:key_step2}
        \Delta V_k \leq \frac{1}{\N_k}\{-\left(L_k(\theta_{k+1})-L_k(\theta^*)\right)-\frac{\gamma\beta\mu}{4}V_k\}\leq 0.
    \end{equation}
    This establishes that V is a Lyapunov function.
    
    Collecting $\Delta V_k$ terms from $t_0$ to $T$, and letting $T\rightarrow\infty$, it can be seen that $L_k(\theta_{k+1})-L_k(\theta^*)\in\ell_1\cap\ell_\infty$ and therefore $\lim_{k\rightarrow\infty}L_k(\theta_{k+1})-L_k(\theta^*)=0$. 
    Furthermore, from the bound on $\Delta V_k$, $V_{k+1}\leq\left(1-\mu\frac{\gamma\beta}{4\N_k}\right)V_k$. Finally, collecting terms, and for constant regressors, $\N_k=\N$:
    \begin{equation}\label{eq:exponential}
        V_{k}\leq\left(1-\mu\frac{\gamma\beta}{4\N}\right)^kV_0\leq \exp\left(-\mu Ck\right)V_0.
    \end{equation}
\end{proof}
\begin{remark}
Very similar to the proof of Theorem \ref{theorem:HOT_smooth}, here too, we employed properties of convex functions. The main distinction between the two theorems is the strong convexity of $L$. This property allows us to obtain a bound for $\Delta V_k$ as in \eqref{eq:key_step2}, and therefore \eqref{eq:exponential}.
\end{remark}

\subsubsection{Accelerated learning for constant regressors}
Since $\phi_k\equiv \phi$, it can be shown that in the quadratic minimization problem, the underlying gradient of $f$ in \eqref{eq:augmented_function} is linear in $\theta$ and therefore satisfies the property
\begin{equation}\label{eq:superposition}
    a\nabla f\left(\theta_k\right)+b\nabla f \left(\theta_{k-1}\right)=\nabla f\left(a\theta_k+b\theta_{k-1}\right)
\end{equation}
for any constants $a$ and $b$. Together with the hyperparameters $\gamma$ and $\beta$ chosen as $\bar{\beta}=1-\beta$ and $\bar{\alpha}=\gamma\beta$, the property \eqref{eq:superposition} allows Algorithm \eqref{alg:HOT_1} along with \eqref{eq:augmented_function} to be reduced to a form
\begin{equation} \label{eq:Nesterov_constant}
    \begin{split}
        \theta_{k+1}&=\nu_k-\bar{\alpha}\nabla f(\nu_k),\\
        \nu_{k+1}&=(1+\bar{\beta})\theta_{k+1}-\bar{\beta}\theta_k.
    \end{split}
\end{equation}
Equation \eqref{eq:Nesterov_constant} coincides with Nesterov's algorithm for strongly convex functions \cite[Equation~2.2.22]{Nesterov2018LecturesOptimization}. This in turn allows us to derive the following accelerated convergence property for the case when $L$ is quadratic in $\theta$:
\begin{theorem}[Modified from {\cite[Theorem~3.18]{Bubeck2015ConvexComplexity}}]\label{theorem:Nesterov_constant}
    For a $\bar{L}$-smooth and $\mu$-strongly convex function $f$, the iterates $\{\theta_k\}^\infty_{k=0}$ generated by \eqref{eq:Nesterov_constant} with $\theta_0=\nu_0$, $\bar{\alpha}=\frac{1}{\bar{L}}$, $\kappa = \bar{L}/\mu$, and $\bar{\beta}=(\sqrt{\kappa}-1)/(\sqrt{\kappa}+1)$, satisfy $f(\theta_k)-f(\theta^*)\leq \frac{\bar{L}+\mu}{2}\lVert\theta_0-\theta^*\rVert^2\exp\left(-\frac{k}{\sqrt{\kappa}}\right)$, and therefore if $k\geq\left\lceil\sqrt{\kappa}\log\left(\frac{(\bar{L}+\mu)\lVert\theta_0-\theta^*\rVert^2}{2\epsilon}\right)\right\rceil$ then $f(\theta_k)-f(\theta^*)\leq\epsilon$.
\end{theorem}
\begin{proof}
    Refer to \cite[Page~290-293]{Bubeck2015ConvexComplexity}.
\end{proof}

For the general case of a strongly convex function $L$, the superposition property \eqref{eq:superposition} is no longer valid. Therefore Nesterov's method of \textit{estimating sequences} is no longer adequate. The convexity property in Definition \ref{def:a2def} can be suitably leveraged to lead to the following inequality \cite{karimi2016unified}:
\begin{equation}\label{eq:convex}
    \mu \lVert x-y\rVert\leq\lVert\nabla f(x)-\nabla f(y)\rVert\leq\bar{L} \lVert x-y\rVert,
\end{equation}
Based on \eqref{eq:convex}, we  propose the following conjecture on an accelerated convergence property of Algorithm \ref{alg:HOT_2}.
\begin{conjecture}\label{conjecture:HT_fast}
For a contnuously differentiable smooth and strongly convex function $\Bar{L}$, and for constant regressors, the iterates $\{\theta_k\}_{k=0}^\infty$ generated by Algorithm \ref{alg:HOT_2} satisfy a convergence rate of $\mathcal{O}(\log(1/\epsilon))$.
\end{conjecture}

In summary, the main results of the paper can be found in Sections \ref{section:linear} and \ref{section:SmoothCase} in the form of Theorems \ref{theorem:HT_continuous}, \ref{theorem:HOT_smooth}, and \ref{theorem:HOT_strongly}. Theorems \ref{theorem:HT_continuous} and \ref{theorem:HOT_smooth} demonstrated the stability property of HT when the underlying loss function was smooth and convex, while Theorem \ref{theorem:HOT_strongly} presented the stability property of HT when the loss function is smooth and strongly convex.

\section{Numerical Simulations}\label{section:sims}
In this section, we numerically validate the results of Theorems \ref{theorem:HOT_smooth} and \ref{theorem:HOT_strongly}, which correspond to convex and strongly convex functions in discrete-time. We also validate Conjecture \ref{conjecture:HT_fast} in Section \ref{section:SmoothCase}, which pertains to an accelerated convergence property of the HT. All simulations have been implemented in Python and the code is available in the online notebooks  \cite{link1} and \cite{link2}.

The starting point for our numerical experiments is a smooth convex function defined as
\begin{equation}\label{eq:simulation_loss}
    L_k(\theta)=\log{(a_ke^{b_k\theta}+a_ke^{-b_k\theta})},
\end{equation}
where $a_k$ and $b_k$ are positive scalars and may be time-varying. This function has a unique minimum at $\theta^*=0$. The  gradient for this loss function can be computed as
$    \nabla L_k(\theta)=(b_k(e^{2b_k\theta}-1))/(e^{2b_k\theta}+1).$
It is also easy to see that the Hessian is upper-bounded as $ \nabla^2 L_k(\theta)\leq b_k^2.$%
  

\subsection{Stability for time-varying regressors}
We represent the time-varying regressors in two ways: by a step change in $b_k$ from $7$ to $14$ at a particular iteration in Figure \ref{fig:simulation_1} and by a sinusoidal change in $b_k=14+7\sin(200k)$ in Figure \ref{fig:simulation_1bis}. The normalizing signal $H_k$ is chosen as $b_k^2$, 
a conservative choice, as opposed to \eqref{eq:H_k_definition}.

In Figures \ref{fig:1a} and \ref{fig:1c} the hyperparameter for both smooth methods is chosen as $\bar{\alpha}=1/\Bar{L}_0$. For the High-order Tuner, we choose $\beta=0.1$ and $\gamma=1/\beta$, so that the effective step sizes are comparable for both methods. These experiments show the stable behavior of HT when the regressors change and instability for the other two methods. In Figure \ref{fig:1b} and \ref{fig:1d} the hyperparameters of the HT are chosen according to Theorem \ref{theorem:HOT_smooth}, i.e.: $\beta=0.1$, $\gamma = \frac{\beta(2-\beta)}{8+\beta}$; and the step size $\Bar{\alpha}$ is chosen as $\Bar{\alpha}=\gamma\beta/\N_0$. Because of the reduction of the effective step size in this experiment, we run more iterations. In this case, all methods remain stable. Nevertheless, since the HT updates its effective step size when the regressors change, it automatically improves its learning rate compared to the other methods.

\begin{figure}
\subfigure[]{\includegraphics[width=0.45\textwidth]{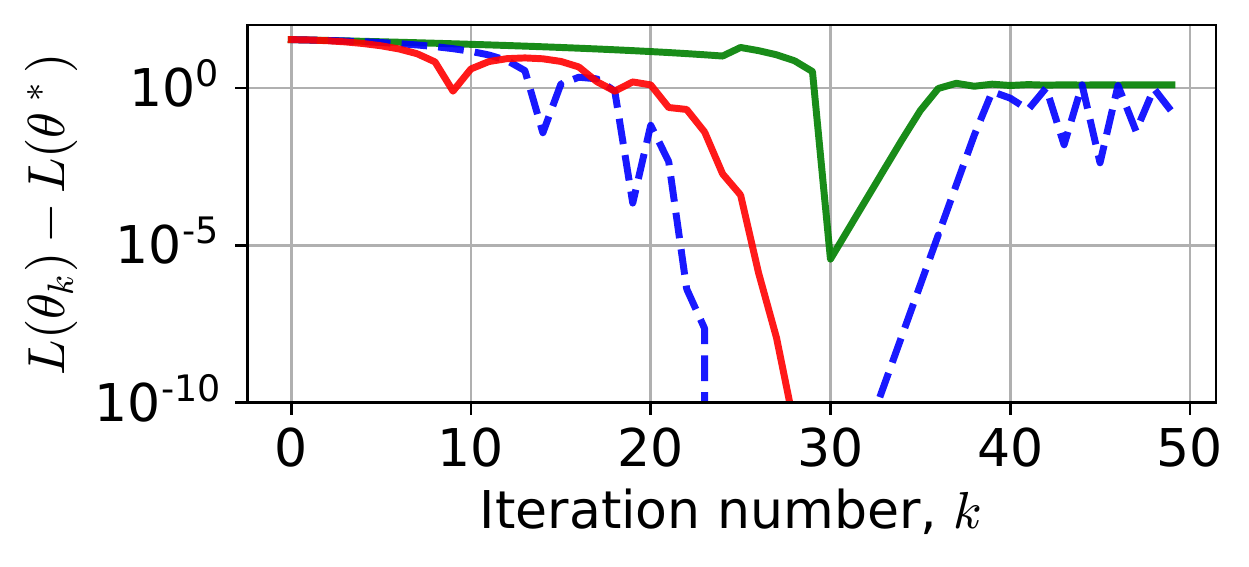}
\label{fig:1a}}
\hfill
\subfigure[]{\includegraphics[width=0.45\textwidth]{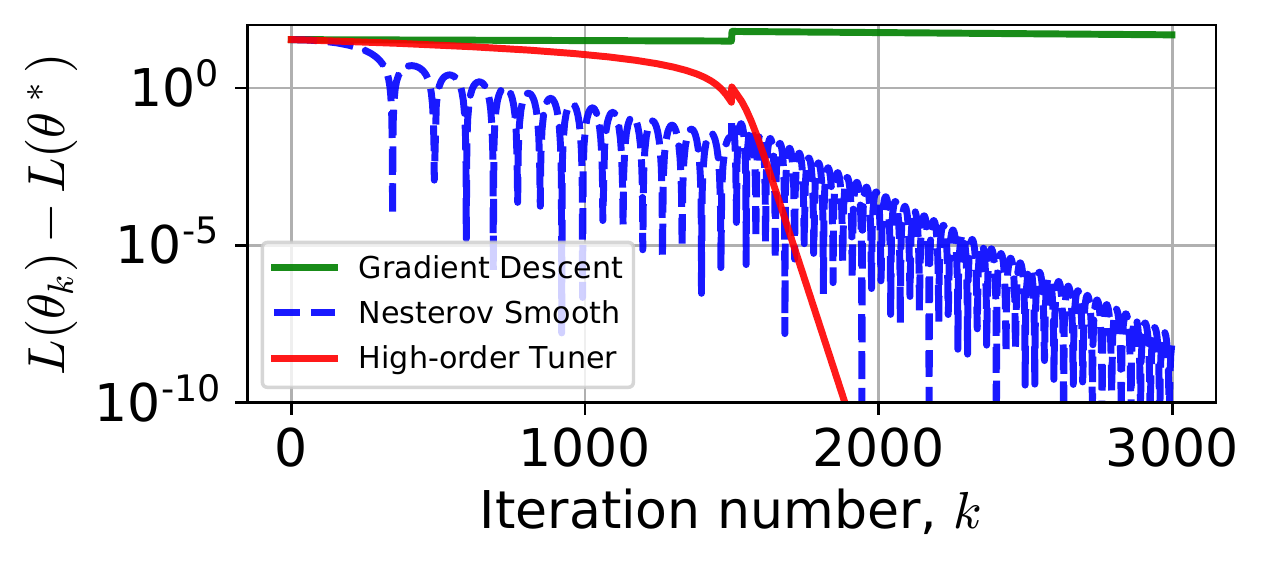}
\label{fig:1b}}
\caption{Stability for time-varying regressors. (a) Step change in $b_k$ from 7 to 14 at $k=25$. (b) Step change in $b_k$ from 7 to 14 at $k=1500$. }
\label{fig:simulation_1}
\end{figure}

\begin{figure}[h]
\subfigure[]{\includegraphics[width=0.44\textwidth]{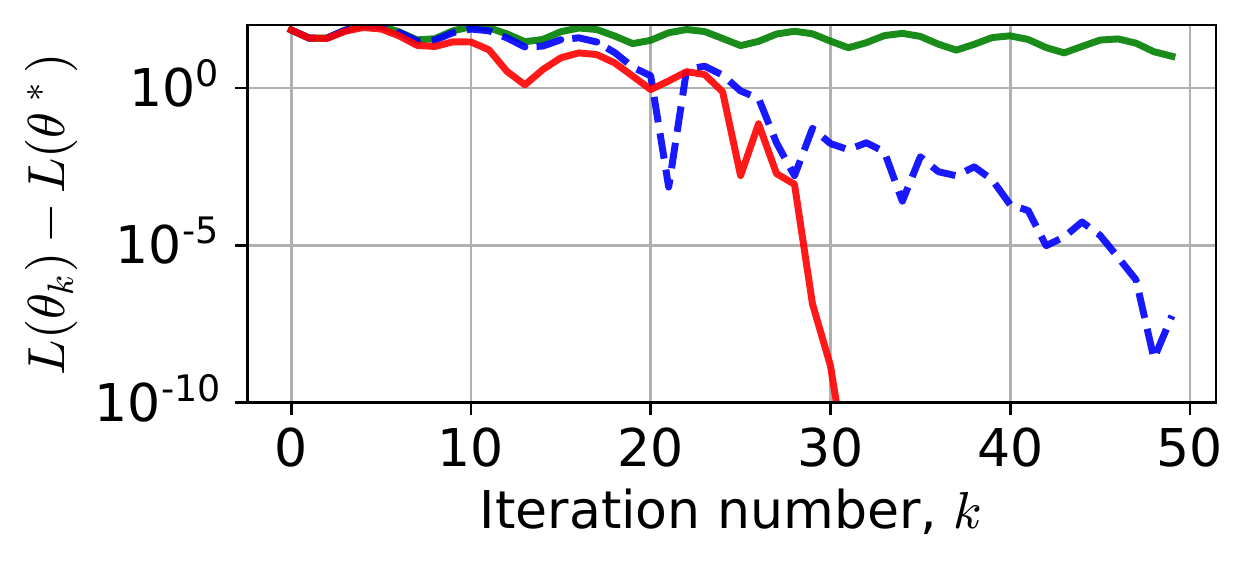}
\label{fig:1c}}
\hfill
\subfigure[]{\includegraphics[width=0.45\textwidth]{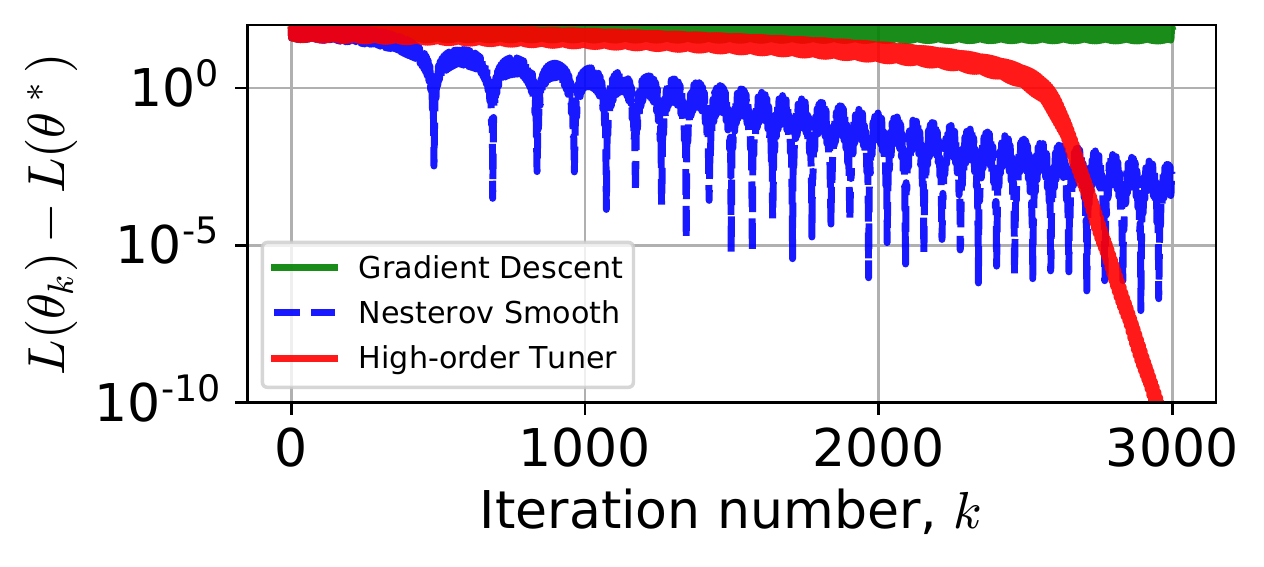}
\label{fig:1d}}
\caption{Stability for time-varying regressors with a sinusoidal change in $b_k$.}
\label{fig:simulation_1bis}
\end{figure}

\subsection{Convergence rate for smooth and strongly convex functions for constant regressors}
First, in order to make \eqref{eq:simulation_loss} strongly convex, we include a regularizing term to the loss function, producing a new smooth and $\mu$-strongly convex function $L_\mu$  as in
\begin{equation}\label{eq:simulation_loss_mu}
    L_\mu(\theta)=\log{(a_ke^{b_k\theta}+a_ke^{-b_k\theta})}+\frac{\mu}{2}\lVert\theta-\theta_0\rVert^2.
\end{equation}

It is clear from the structure of \eqref{eq:simulation_loss_mu} that the underlying problem is nonlinearly parameterized, even while the function remains strongly convex and smooth. Because of the nonlinearity of the gradient of $L$, Algorithm \ref{alg:HOT_2} cannot be reduced to Nesterov's method \eqref{eq:Nesterov_constant}. Thus, the accelerated convergence rate for the loss function in Theorem \ref{theorem:Nesterov_constant} cannot be directly extended to this case. 
\begin{figure}[h]
    \centering
    \includegraphics[width=0.45\textwidth]{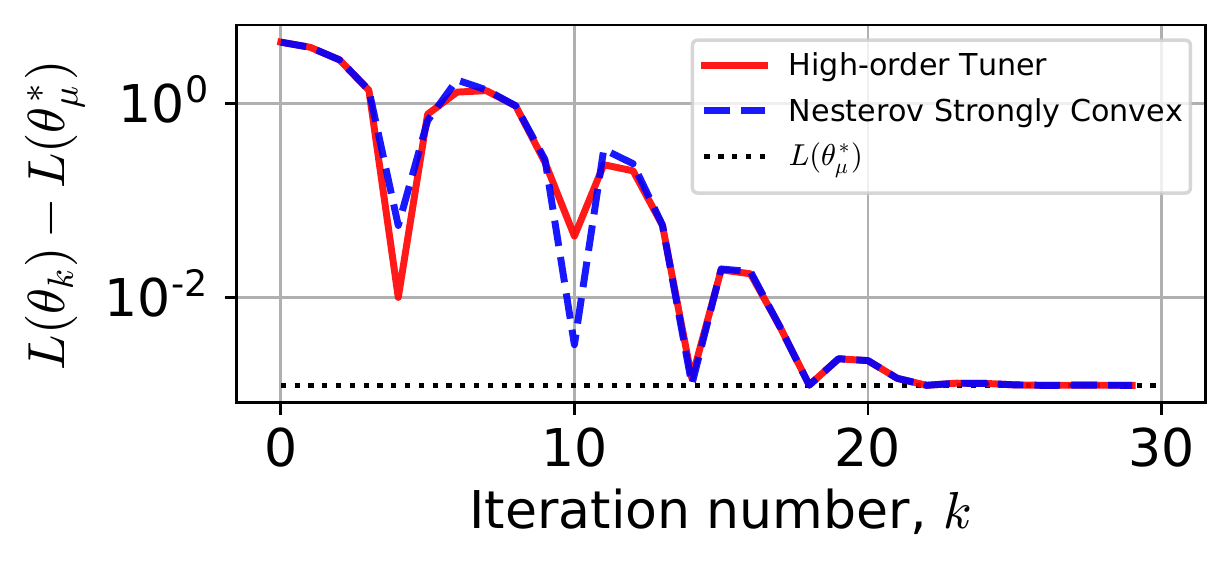}
    \caption{Accelerated learning properties of the HT. Original smooth loss function gap values at each iteration, and loss function value at $\theta_\mu^*$.}
    \label{fig:simulation}
\end{figure}
However, as shown in Figure \ref{fig:simulation}, our numerical simulation studies show that when Algorithm \ref{alg:HOT_2} is applied to the convex loss function \eqref{eq:simulation_loss_mu}, it results in the same accelerated convergence rate as  Nesterov's algorithm applied to $\Bar{f}_\mu(\theta)=L_\mu(\theta)/\N$. Figure \ref{fig:simulation} shows the results when \eqref{eq:Nesterov_constant} and Algorithm \ref{alg:HOT_2} minimize function $\bar{f}_\mu$  and \eqref{eq:simulation_loss_mu} respectively, with $a_k\equiv\frac{1}{2}$, $b_k\equiv1$, $\mu=10^{-4}$ and $\theta_0=5$. The hyperparameters for this simulation have been chosen as in Theorem \ref{theorem:Nesterov_constant} and  $\beta = 1-\bar{\beta}$ and $\gamma=\bar{\alpha}/\beta$.  It can be seen that both algorithms result in an equally fast convergence supporting Conjecture \ref{conjecture:HT_fast}.

\section{Summary}
In this paper we have shown that the discrete HT proposed in \cite{Gaudio2020AcceleratedRegressors} is also stable for general convex loss functions. An added advantage of this HT is an accelerated convergence of the loss function to zero, as shown in the numerical simulations. We anticipate that the tools used to prove these results, as well as the insights gained from this framework, will be useful for further extensions to problems with nonlinear regression and nonconvexities.


\bibliographystyle{IEEEtran}
\bibliography{moreu_cdc21}

\end{document}